\documentclass[letterpaper]{article}

\usepackage{uai2020}
\usepackage[margin=1in]{geometry}
\usepackage{times}

\usepackage[round]{natbib}
\usepackage{graphicx}
\usepackage{tikz}
\usetikzlibrary{arrows,decorations.pathmorphing,backgrounds,positioning,fit}

\usepackage[hidelinks]{hyperref}
\usepackage{algorithm}
\usepackage{algorithmic}

\usepackage{amsmath}
\usepackage{amssymb}
\usepackage{amsthm}

\newcommand{\tr}{\operatorname{tr}} 
\newcommand{\sign}{\operatorname{sign}}

\newtheorem{thm}{Theorem}[section]
\newtheorem{prop}[thm]{Proposition}

\newtheorem{cor}[thm]{Corollary}

\title{Graphical continuous Lyapunov models}
\author{ {\bf Gherardo Varando} \\
Department of Mathematical Sciences\\
University of Copenhagen \\
Copenhagen, Denmark \\
\And
{\bf Niels Richard Hansen}   \\
Department of Mathematical Sciences\\
University of Copenhagen \\
Copenhagen, Denmark \\
}

\begin{document}
\maketitle
\begin{abstract}
	      The linear Lyapunov
        equation of a covariance matrix parametrizes the 
        equilibrium covariance matrix of a stochastic
        process. This parametrization can be interpreted as
				a new graphical model class, and we show 
				how the model class behaves under marginalization
        and introduce a method for structure learning via
        $\ell_1$-penalized loss minimization. Our proposed method is
        demonstrated to outperform alternative structure learning
        algorithms in a simulation study, and we illustrate its
        application for protein phosphorylation network reconstruction.
\end{abstract}

\section{INTRODUCTION}

Path analysis as introduced by \cite{wright1921, wright1934} illustrates how
covariance computations in linear models can benefit from a
graphical model representation. Today there is a vast literature on
linear structural equation models and their corresponding algebraic
and graphical model theory, see e.g. \cite{drton2018}. Within this framework, the standard 
parametrization specifies the covariance matrix $\Sigma$ as a solution to the equation
\begin{equation} \label{eq:SEM}
(I - \Lambda)^T \Sigma (I - \Lambda) = \Omega
\end{equation} 
for matrix parameters $\Lambda$ and $\Omega$.
The associated mixed graph has directed edges and bidirected edges
determined by the nonzero entries of $\Lambda$ and $\Omega$,
respectively. If we fix an acyclic graph, say, the framework provides a
parametrization of the observables from a directed acyclic model
-- potentially with latent variables -- see \citep{Richardson:2002}. In the
cyclic case the parametrization can, moreover, be interpreted as an equilibrium
distribution for a deterministic process whenever the spectrum of
$\Lambda$ is inside the unit circle, see e.g. \citep{Hyttinen:2012}.

It is, however, well known that for certain continuous time
stochastic processes the equilibrium covariance matrix does not have a
simple graphical representation using the parametrization above, see e.g.
\citep{Mogensen:2018uai}. Instead it has an alternative 
parametrization corresponding to the graphical
representation of the dynamics of the process. In this
parametrization, $\Sigma$ is the solution to the continuous Lyapunov equation,
\begin{equation}
	\label{eq:CLE}
	B \Sigma + \Sigma B^T  + C = 0   
\end{equation}
where $B$ and $C$ are matrices parametrizing $\Sigma$.

Models given by \eqref{eq:CLE} are of practical interest when 
only cross-sectional data from the stochastic process can be
obtained. This is the case for biological systems such as gene
regulatory or protein signalling networks, where cells are destroyed in
the measurement process. Existing methods based on correlation or
mutual information, such as the ARACNe method by \cite{basso2005}, the
use of directed graphical models, \citep{sachs2005}, or
the graphical lasso giving undirected graphs, \citep{friedman2007},
cannot represent feedback processes, whereas cycles can be encoded
naturally by \eqref{eq:CLE}.



The main objective of this paper is to develop the framework
of graphical models parametrized by \eqref{eq:CLE} and to introduce a
learning algorithm of the graphical structure. In the preparation of
this paper we found that similar ideas were recently considered by
\cite{young2019} and \cite{fitch2019}. The work by \cite{fitch2019}
is based on \eqref{eq:CLE} and a learning
algorithm was proposed, while \cite{young2019} considered
the vector autoregressive model, whose equilibrium covariance matrix
solves the \emph{discrete} Lyapunov equation. 

We connect in this paper the models parametrized by \eqref{eq:CLE} 
to the concept of local independence for stochastic processes, and we  
present new results about these models as graphical models. To 
this end, recall that 
Wright's path analysis lead to polynomial expressions of the entries in 
$\Sigma$ in terms of the nonzero entries in $\Lambda$ and $\Omega$. 
Such formulas are in modern terminology
known as \emph{trek rules}, and they explain
how graphical structural constraints are encoded into $\Sigma$. By 
introducing trek seperation,
\cite{sullivant:2010} gave, for instance, a complete graph-theoretic characterization
in the acyclic case of when submatrices of $\Sigma$ will drop
rank. Another example is the half-trek criterion for generic
identifiability by \cite{foygel2012}.

In this paper we associate a mixed graph to the covariance matrix
solving \eqref{eq:CLE} and establish a version of trek rules when $B$
is a stable matrix. We use this to introduce a novel graphical
projection yielding a parametrization of marginalized models in terms
of solutions to Lyapunov equations. To fit models parametrized by
\eqref{eq:CLE}, but with an unknown graphical structure, we propose
$\ell_1$-penalized loss minimization using either the Frobenius norm
or the Gaussian log-likelihood loss. They outperformed the learning algorithm proposed by
\cite{fitch2019} in a simulation study, and we illustrate the use of 
the method for protein phosphorylation network discovery using data 
from \cite{sachs2005}.

\section{GRAPHICAL CONTINUOUS LYAPUNOV MODELS}

We will consider models of covariance matrices determined as solutions
to the Lyapunov equation \eqref{eq:CLE} and parametrized by the
matrices $B$ and $C$. Note that \eqref{eq:CLE} can be written in 
tensor product form as the linear equation
$$(B \otimes I + I \otimes B) \mathrm{vec}(\Sigma) = - \mathrm{vec}(C).$$
The eigenvalues of the \emph{kronecker sum} $B \otimes I + I \otimes B$
are sums of pairs of eigenvalues of $B$, \citep[Theorem 4.4.5]{Horn:1991}.
The solution to \eqref{eq:CLE} is thus unique if and only if 
the sum of any two eigenvalues of $B$ is nonzero, 
in which case $\Sigma(B,C)$ will denote the unique solution. 

Some notation and terminology is needed to study solutions of
\eqref{eq:CLE}. Introduce $\mathrm{Mat}_0(p)$ as the set of $p \times p$
matrices that do not have two eigenvalues summing to zero, and
let $\mathrm{Sym}(p)$ denote the set of symmetric $p \times p$ matrices. Let
$\mathrm{Stab}(p)$ denote the set of stable $p
\times p$ matrices, that is, matrices whose eigenvalues all have a strictly
negative real part. Obviously, $\mathrm{Stab}(p)
\subseteq \mathrm{Mat}_0(p)$. The set of $p \times p$ positive
definite matrices is denoted $\mathrm{PD}(p)$. 

The sparsity patterns of the parameters $B$ and $C$ will be encoded via a 
mixed graph, that is, a graph $\mathcal{G} = ([p], E)$ with vertices $[p] =
\{1, \ldots, p\}$ and with $E$ containing directed as well as bidirected edges. 
Self loops and multiple edges between two nodes are allowed. 
We say that a pair of matrices $(B,C) \in \mathrm{Mat}_0(p) \times 
\mathrm{Sym}(p)$ are compatible with a mixed graph 
$\mathcal{G}$ if $B_{ji} \neq 0 $ implies $i \rightarrow j$ and $C_{ij} \neq 0$ implies
$i \leftrightarrow j$. The set of $\mathcal{G}$-compatible matrix
pairs is denoted $\Xi_{\mathcal{G}} \subseteq \mathrm{Mat}_0(p) \times 
\mathrm{Sym}(p)$, and $\Theta_{\mathcal{G}} = \Xi_\mathcal{G} \cap 
\left( \mathrm{Stab}(p) \times \mathrm{PD}(p) \right).$

Given a mixed graph $\mathcal{G}$, the map $(B, C) \mapsto \Sigma(B, C)$ is
well defined on $\Xi_{\mathcal{G}}$ with image in
$\mathrm{Sym}(p)$. The restriction of this map to
$\Theta_{\mathcal{G}}$ has image in $\mathrm{PD}(p)$, which follows
from Proposition \ref{prop:intrep} below. Let $\mathcal{M}_{\mathcal{G}} =
\Sigma(\Theta_{\mathcal{G}}) \subseteq \mathrm{PD}(p)$ denote the
image of $\Theta_{\mathcal{G}}$, which we
call the \emph{graphical continuous Lyapunov model} (GCLM) with graph 
$\mathcal{G}$. The \emph{extended} GCLM is $\mathcal{M}_{\mathcal{G}}^e =
\Sigma(\Xi_{\mathcal{G}})$.

\subsection{STOCHASTIC PROCESSES AND LOCAL INDEPENDENCE}
\label{sec:li}

To motivate \eqref{eq:CLE} consider the $p$-dimensional
Ornstein-Uhlenbeck process given as a solution to the stochastic
differential equation 
\begin{equation} \label{eq:OU}
	dX_t = B(X_t - a)dt + D dW_t  
\end{equation} 
where $B$ and $D$ are $p \times p$ matrices, $a \in \mathbb{R}^p$ and 
$W_t$ is a standard Brownian motion in $\mathbb{R}^p$. If $B$ is a
stable matrix, \eqref{eq:OU} has a Gaussian equilibrium 
distribution with covariance matrix $\Sigma(B,DD^T)$, see e.g. 
\citep[Theorem 2.12]{Jacobsen:1992}. Thus solutions
of \eqref{eq:CLE} arise as equilibrium covariances for continuous
time stochastic processes.  

We call \eqref{eq:OU} a structural causal stochastic differential
equation if it adequately captures effects of interventions, see
\citep{Sokol:2014}. In this case the directed part of the mixed graph
$\mathcal{G}$ -- introduced above in terms of $B$ -- represents direct causal
effects. Moreover, if there is no directed edge from $i$ to $j$, the corresponding
coordinates of the stochastic process satisfy an infinitesimal
conditional independence, and we say that $X_t^j$ is locally
independent of $X_t^i$. The directed part of $\mathcal{G}$ is, by Definition 12 in
\cite{Mogensen:2018uai}, also identical to the local independence graph
determined by \eqref{eq:OU}.  

If $C = DD^T$ is diagonal, the local independence graph has the global Markov
property for local independence, see \cite{Mogensen:2018uai},  who also gave a
learning algorithm for partially observed systems. That general algorithm learns
an equivalence class of local independence graphs  by local independence
queries. In the specific case of solutions to \eqref{eq:OU}, the equilibrium
covariance matrix also carries information about the local  independence graph
as encoded via the Lyapunov equation. As we will show below, graphical
representations of the marginalization of the equilibrium covariance matrix
requires  a new graphical projection that introduces additional bidirected
edges,  but in any case, at least for diagonal $C$, the directed edges of
$\mathcal{G}$ have an interpretation as local dependences -- and even direct
causal effects if \eqref{eq:OU} is a structural causal stochastic differential
equation. 

\subsection{TREKS}

To obtain a graphical representation of $\Sigma = \Sigma(B, C) \in \mathcal{M}_{\mathcal{G}}$ for
a mixed graph $\mathcal{G}$ we introduce 
\begin{equation} \label{eq:intrepPartial}
\Sigma(s) = \int_0^{s} e^{uB} C
e^{uB^T} \mathrm{d} u.
\end{equation}
The following is a well known result, see \citep{Jacobsen:1992}
or \cite[Theorem 2]{fitch2019}, but we include it for
completeness. 

\begin{prop} \label{prop:intrep}
  For $(B, C) \in \Theta_{\mathcal{G}}$
\begin{equation} \label{eq:intrep}
\Sigma(B, C) = \lim_{s \to \infty} \Sigma(s)  = \int_0^{\infty} e^{uB} C
e^{uB^T} \mathrm{d} u.
\end{equation}
\end{prop}

\begin{proof} First note that stability of $B$ ensures that
  the solution to the Lyapunov equation is unique. It also ensures
  that the integral in \eqref{eq:intrep} is convergent. We see that if
  $\Sigma$ is given by  the r.h.s. of \eqref{eq:intrep}  then
\begin{align*}
B \Sigma + \Sigma B^T & = \int_0^{\infty} B e^{uB} C
e^{uB^T} + e^{uB} C
e^{uB^T} B^T \mathrm{d} u \\
& = \int_0^{\infty} \frac{\mathrm{d}}{\mathrm{d}u} e^{uB} C
e^{uB^T} \mathrm{d} u = - C,
\end{align*}
which shows that $\Sigma$ solves \eqref{eq:CLE}. 
\end{proof}

The representation \eqref{eq:intrep} implies that 
$\Sigma$ is positive definite if $C$ is, which shows that $\mathcal{M}_{\mathcal{G}}
\subseteq \mathrm{PD}(p)$ as claimed above. 

A \emph{trek} from $i$ to $j$, denoted $i \leadsto j$, is a walk of the form 
$$\tau: \ \underbrace{i \leftarrow \cdots \leftarrow i_1}_{n(\tau)} 
\leftarrow k \leftrightarrow l \rightarrow \underbrace{j_1
  \rightarrow \cdots \rightarrow j}_{m(\tau)}$$
where $k, l \in [p]$ are connected by a bidirected edge. Thus a trek 
consists of a left hand side, which is a directed walk 
$k \rightarrow i_1 \rightarrow \ldots \rightarrow i$ of length 
$n(\tau)$, and a right hand side, which is a directed walk  
$l \rightarrow j_1 \rightarrow \ldots \rightarrow j$ of 
length $m(\tau)$. Those two walks are connected by the bidirected 
edge $k \leftrightarrow l$. For every trek $i \leadsto j$ there is 
a reversed trek, $j \leadsto i$, corresponding to interchanging 
the roles of the left and right hand sides of the trek. Note that $n(\tau) = 0$
with $i = k$ as well as $m(\tau) = 0$ with $j = l$ are allowed. Define also 
$$\kappa(s, \tau) = \frac{s^{(n(\tau) + m(\tau) + 1)}}{(n(\tau) + m(\tau) + 1)n(\tau)!m(\tau)!}$$
for any trek $\tau$ and $s \in \mathbb{R}$, and introduce  
for $(B, C) \in \Theta_{\mathcal{G}}$  and a trek $\tau$ the \emph{trek weight}
$$\omega(B, C, \tau) = C_{k, l} \prod_{g \rightarrow h \in
  \tau} B_{hg}.$$

\begin{prop} \label{prop:trek}
For $(B, C) \in \Theta_{\mathcal{G}}$ 
$$\Sigma(s)_{ij} = \sum_{\tau \in \mathcal{T}(i,j)} \kappa(s, \tau)
\omega(B, C, \tau)$$
where $\mathcal{T}(i,j)$ denotes the set of all treks from $i$ to $j$. 
\end{prop}

\begin{proof} Using the series expansion of the matrix exponential we
  find that 
{\small
\begin{align*}
\Sigma(s)_{ij} & = \int_0^{s} \sum_{n=0}^{\infty} \sum_{m=0}^{\infty}
                          \sum_{k,l=1}^p 
                          \frac{t^n t^m}{n!m!} (B^n)_{ik}
                          C_{kl} (B^m)_{jl} \mathrm{d} t \\
& = \sum_{n=0}^{\infty} \sum_{m=0}^{\infty} \sum_{k,l=1}^p 
                          \frac{s^{(n + m + 1)}}{(n + m + 1)n!m!} (B^n)_{ik}
                          C_{kl} (B^m)_{jl} \\
& = \sum_{\tau \in \mathcal{T}(i,j)} \kappa(s, \tau) \omega(B, C,
  \tau). \qedhere
\end{align*}
}
\end{proof}

The following corollary is an immediate consequence of Propositions
\ref{prop:intrep} and \ref{prop:trek}.

\begin{cor} \label{cor:trek} If $\Sigma \in \mathcal{M}_{\mathcal{G}}$ and there is
  no trek from $i$ to $j$ in $\mathcal{G}$ then $\Sigma_{ij} = 0$. 
\end{cor}

\subsection{MARGINALIZATION}
\label{sec:marginalization}
Let $\Sigma$ be a $p' \times p'$ matrix that solves
the Lyapunov equation for given $B$ and $C$, and suppose that we only observe
variables corresponding to the top left $p \times p$ block, $\Sigma_{11}$, for $p < p'$. 
Writing out the Lyapunov equation in block matrix form gives four
coupled equations. The one corresponding to $\Sigma_{11}$ is the Lyapunov equation  
\begin{equation}
\label{eq:LyapunovMarg}
B_{11} \Sigma_{11} + \Sigma_{11}
B_{11}^T + \tilde{C} = 0
\end{equation}
with $\tilde{C} = B_{12} \Sigma_{21} + \Sigma_{12} B_{12}^T + C_{11}.$

When $C$ is symmetric so is $\tilde{C}$, but there is 
no guarantee that it is positive definite even if $C$ is so, nor that $B_{11}$ is
stable if $B$ is so. What we can show is that if
$\Sigma$ is a GCLM then $\Sigma_{11}$ is an extended GCLM. To do so we
will introduce a graphical projection map.

For $\mathcal{G} = ([p'], E)$ a mixed graph let $\mathcal{G}[p] =
([p], E[p])$ denote the projection onto the first $p < p'$ vertices defined as follows: 
for $i,j \in [p]$
\begin{itemize}
\item $i \rightarrow j \in E[p]$ if $i \rightarrow j \in E$ 
\item $i \leftrightarrow j \in E[p]$ if $i \leftrightarrow j \in E$
\item $i \leftrightarrow j \in E[p]$ if for some $k >
  p$ there is a trek from $i$ to $j$ of the forms $i \leftarrow
  k \leadsto j$ or  $i \leadsto k \rightarrow j$
\end{itemize}
Thus the projected graph retains all edges in $\mathcal{G}$ between
vertices in $[p]$. In addition, it has bidirected edges between
vertices $i,j \in [p]$ that are connected by a trek containing a
vertex not in $[p]$, which is directly connected to either $i$ or $j$ in
the trek. It should be noted that this \emph{is not} a standard latent
graph projection. For once, only bidirected arrows are added.

\begin{prop} If $\Sigma \in \mathcal{M}_{\mathcal{G}}$ and
  $B_{11} \in \mathrm{Mat}_0$ then 
$\Sigma_{11} \in \mathcal{M}^e_{\mathcal{G}[p]}$.
\end{prop}

\begin{proof} It is clear from the definitions that $B_{11}$
  fulfills the $\mathcal{G}[p]$-compatibility requirement. Observe
  then that 
	$$\tilde{C}_{ij}  = C_{ij} + \sum_{k = p + 1}^{p'} \left( B_{ik} 
	\Sigma_{kj} + \Sigma_{ik} B_{jk}\right),$$ 
	which is symmetric in $i$ and $j$. If
$C_{ij} \neq 0$ then $i \leftrightarrow j$. If
$\tilde{C}_{ij} \neq 0$, but $C_{ij} = 0$, then there is a $k > p$
such that $B_{ik} \Sigma_{kj} \neq 0$ or $\Sigma_{ik} B_{jk} \neq
0$. In the first case this means that $\Sigma_{kj} \neq 0$, and by
Corollary \ref{cor:trek} there is a trek from $k$ to $j$. Now as
$B_{ik} \neq 0$ as well, we can extend the trek to the left with the
edge $k \to i$, and $i \leftrightarrow j$ by the definition of
$\mathcal{G}[p]$. A similar argument applies if
$\Sigma_{ik} B_{jk} \neq 0$.

In conclusion, $(B_{11}, \tilde{C})$ is
$\mathcal{G}[p]$-compatible, and since it is assumed that
$B_{11} \in \mathrm{Mat}_0$ we have that 
\begin{equation*}
\Sigma_{11} = \Sigma(B_{11}, \tilde{C}) \in
\mathcal{M}^e_{\mathcal{G}[p]}. \qedhere
\end{equation*} 
\end{proof}



\begin{figure}
\begin{tikzpicture}
\tikzset{vertex/.style = {shape=circle,draw,minimum size=1.5em, inner sep = 0pt}}
\tikzset{vertexdot/.style = {shape=circle,fill=red,color=red,draw,minimum size=0.5em, inner sep = 0pt}}
\tikzset{vertexhid/.style = {shape=rectangle,draw,minimum size=1.5em, inner sep = 0pt}}
\tikzset{edge/.style = {->,> = latex', thick}}
\tikzset{edgeun/.style = {-, thick}}
\tikzset{edgebi/.style = {<->,> = latex', thick}}

\node at (0, 2.5) {(A)};
\node at (4, 2.5) {(B)};

\node[vertex] (A) at (-1.2, 1.2) {$1$};
\node[vertex] (B) at (1.2, 1.2) {$2$};
\node[vertex] (C) at (1.2, -1.2) {$3$};
\node[vertex] (D) at (-1.2, -1.2) {$5$};
\node[vertex] (E) at (0, 0) {$4$};

\draw[edge] (A) edge[bend right  = 20, color=blue] node[above, color=black, scale=0.7] {$-1$} (B);
\draw[edge] (A) edge[loop left, color=blue] node[left, color=black, scale=0.7] {$-1$} (A);
\draw[edge] (B) edge[bend right  = 20, color=blue] node[above, color=black, scale=0.7] {$1$} (A);
\draw[edge] (C) edge[bend right  = 15, color=blue] node[right, color=black, scale=0.7] {$0.2$} (B);
\draw[edge] (C) edge[bend left  = 15, color=blue]  node[above, color=black, scale=0.7] {$1$} (D);
\draw[edge] (C) edge[loop right, color=blue] node[right, color=black, scale=0.7] {\ $-1$} (C);
\draw[edge] (E) edge[loop above, color=blue] node[right, pos = 0.7, color=black, scale=0.7] {\ $-1$} (E);
\draw[edge] (E) edge[bend left  = 15, color=blue]  node[right, pos=0.3, color=black, scale=0.7] {$-0.5$} (C);
\draw[edge] (D) edge[bend left = 15, color=blue]   node[left, pos=0.7, color=black, scale=0.7] {$1$ \ } (E);
\draw[edge] (D) edge[loop left, color=blue] node[left, color=black, scale=0.7] {$-1$} (D);

\draw[edgebi, loop above, color=red] (A) to (A);
\draw[edgebi, loop above, color=red] (B) to (B);
\draw[edgebi, loop below, color=red] (C) to (C);
\draw[edgebi, loop below, color=red] (D) to (D);
\draw[edgebi, loop below, color=red] (E) to (E);

\node[vertex] (A1) at (-1.2+4, 1.2) {$1$};
\node[vertex] (B1) at (1.2+4, 1.2) {$2$};
\node[vertex] (C1) at (1.2+4, -1.2) {$3$};
\node[vertex] (E1) at (4, 0) {$4$};

\draw[edge, loop left, color=blue] (A1) to (A1);
\draw[edge, loop above, color=blue] (E1) to (E1);

\draw[edge, bend right = 20, color=blue] (A1) to (B1);
\draw[edge, bend right  = 15, color=blue] (C1) to (B1);
\draw[edge, bend right  = 20, color=blue] (B1) to (A1);
\draw[edge, bend left  = 15, color=blue] (E1) to (C1);

\draw[edgebi] (E1) edge[bend right = 15, color=red] node[left, pos = 0.8, color=black, scale=0.7] {$0.20$ \ } (C1);
\draw[edgebi] (E1) edge[bend left = 15, color=red] node[left, color=black, scale=0.7] {$0.05$ \ } (A1);
\draw[edgebi] (E1) edge[bend right = 15, color=red] node[right, pos = 0.3, color=black, scale=0.7] {\ $0.07$} (B1);
\draw[edgebi] (A1) edge[loop above, color=white] node[above, color=black, scale=0.7] {$1$} (A1);
\draw[edgebi] (B1) edge[loop above, color=white] node[above, color=black, scale=0.7] {$1$} (B1);
\draw[edgebi] (C1) edge[loop below, color=white] node[below, color=black, scale=0.7] {$1$} (C1);
\draw[edge] (C1) edge[loop right, color=blue] (C1);
\draw[edgebi] (E1) edge[loop below, color=white] node[left, color=black, scale=0.7] {$1.60$ \ \ } (E1);
\draw[edgebi, loop above, color=red] (A1) to (A1);
\draw[edgebi, loop above, color=red] (B1) to (B1);
\draw[edgebi, loop below, color=red] (C1) to (C1);
\draw[edgebi, loop below, color=red] (E1) to (E1);

\end{tikzpicture} 
\caption{Mixed graphs representing a GCLM with $p = 5$ nodes (A) and the extended GCLM
  (B) obtained by marginalization of (A). The larger model (A) has $C
  = I$ and the nonzero entries of $B$ are shown as edge weights for
  the directed edges. The marginalized model (B) has the same directed
  edge weights and the nonzero entries of $\tilde{C}$ are shown as
  edge weights for the bidirected edges. \label{fig:graph0}}
\end{figure}
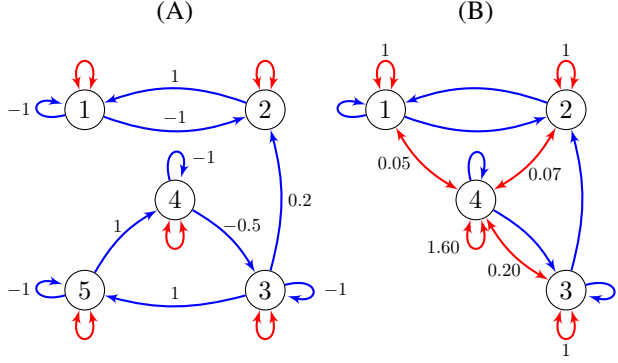

\subsection{EXAMPLE}

Consider the GCLM with $\mathcal{G}$ as given by (A) in Figure \ref{fig:graph0}. In this example $p =
  5$ and the only bidirected edges are self loops. The directed part of 
	$\mathcal{G}$ is the local independence graph of the stochastic process,
	see Section \ref{sec:li}.
	
	The specific model has 
$$
{\small
B = \left(\begin{array}{rrrrr}
-1 & 1 & . & . & . \\ 
  -1 & . & 0.2 & . & . \\ 
  . & . & -1 & -0.5 & . \\ 
  . & . & . & -1 & 1 \\ 
  . & . & 1 & . & -1 \\ 
\end{array}\right)
}
$$
and $C = I_5$ the identity matrix. The eigenvalues of $B$ are 
$$ -1.79,  \ -0.60 \pm 0.69i, \text{ and }
-0.50 \pm 0.87i,$$
with all real parts strictly negative, whence $B$ is stable. The
graphical projection when projecting away node 5 is shown in  Figure
\ref{fig:graph0} (B). The only directed edge out of 5 is $5 \rightarrow 4$, 
and it follows from the projection map that the added bidirected edges are
$4 \leftrightarrow 1$, $4 \leftrightarrow 2$ and $4 \leftrightarrow 3$. 
In this example, $B_{11}$ is, in fact, still a stable matrix, and 
by solving the Lyapunov equation in terms of $B$ and
$C$ the $\tilde{C}$ matrix was computed to be 
$$
{\small
\tilde{C} = \left(\begin{array}{rrrr}
1 & . & . & 0.05 \\ 
  . & 1 & . & 0.07 \\ 
  . & . & 1 & 0.20 \\ 
  0.05 & 0.07 & 0.20 & 1.60 \\ 
\end{array}\right).
}
$$
The graphical projection in Figure \ref{fig:graph0} (B) should be compared to the 
graphical projection of the local independence graph, \citep{Mogensen:2020,
Mogensen:2018uai}, which introduces a directed edge from node 3 to node 4 instead 
of the three bidirected edges. That projection represents local independences 
of the marginalized nodes \citep{Mogensen:2020}. We have not developed 
a notion of separation for the mixed graph in Figure \ref{fig:graph0} (B), and it does not 
represent local independence among the marginalized nodes directly. However,
its representation of the parametrization of the marginalized equilibrium 
covariance matrix allows us to read of direct causal effects among the 
observed nodes when the model of all nodes is a structural causal stochastic 
differential equation.

\section{STRUCTURE RECOVERY}

We propose minimizing an $\ell_1$-penalized loss to estimate 
the directed part of a GCLM as given by the $B$ matrix in
\ref{eq:CLE}. The $C$ matrix will be held diagonal.

Specifically, we suggest estimating $(B,C)$ by solving the following
optimization problem for a generic differentiable loss function 
$L:\mathrm{PD}(p) \to \mathbb{R}$:

\begin{equation} 
\begin{array}{ll}
	\text{minimize}  & L\left(\Sigma(B,C) \right) + 
	\lambda \rho_1(B)  + \kappa ||C - I_p||_F^2 \\
	\text{subject to} &   B \text{ stable and } C \text{ diagonal},
\end{array}
	\label{eq:l1penalB}
\end{equation}
where $\lambda, \kappa \geq 0$ are regularization parameters
and $\rho_1(B) = \sum_{i\neq j} |B_{ij}|$ is the $1$-norm 
 of the off-diagonal entries of $B$. 
 The penalization term involving the Frobenius norm of the 
 difference between $C$ and the identity matrix is necessary, 
 since the pair $(B,C)$ can only be identified up to a multiplicative constant. 
 Letting $\kappa = + \infty$, we obtain as a special case an
 estimator of $B$ with $C=I_p$ fixed. Smaller values of
 $\kappa$ allow for $C$ matrices with diverging diagonal entries.  
 
 Examples of loss functions are the negative
 Gaussian log-likelihood
 \[ 
   \log\det \Sigma + \tr\left(\hat{\Sigma} \Sigma^{-1}\right),
 \] 
 and the squared Frobenius loss 
 \[\| \Sigma - \hat{\Sigma} \|_{F}^2 = 
 \sum_{i,j} \left(\Sigma_{ij} - \hat{\Sigma}_{ij} \right)^2, \] 
 for a given positive semi-definite matrix $\hat{\Sigma}$.

We use a variation of the proximal gradient algorithm for solving
\eqref{eq:l1penalB}, see ~\citep{parikh2014},
even though the optimization problem is in general non-convex. 
The proximal operator for $\ell_1$-penalization is soft-thresholding ($\mathcal{S}_{t}(x) = 
\sign(x)\left(|x| - t \right)$), and each iteration of the algorithm
amounts to  
\begin{align*} 
	C^{(k)} &= C^{(k-1)} - st\kappa (C^{(k-1)} - I_p) \\
	&- s t (\nabla_C L(\Sigma( B^{(k-1)}, C^{(k-1)})))    \\
	B^{(k)} &= \mathcal{S}_{sr\lambda}\left( B^{(k-1)} - sr\nabla_B
	L(\Sigma( B^{(k-1)}, C^{(k-1}))) \right), 
 \end{align*} 
where soft-thresholding of a matrix is defined elementwisely.
The global step size $s$ is chosen using line search as in~\cite{beck2010} 
once the independent steps $t$ and $r$ have been chosen small enough
that $C^{(k)}$ is positive definite and $B^{(k)}$ is stable.

Detailed pseudo-code of our proposed 
proximal gradient based algorithm is given as
Algorithm~\ref{alg:proxgradb}. 


The gradients with respect to $B$ and $C$ can be obtained with the 
cost of solving one additional Lyapunov equation as shown in the 
following proposition.
\begin{prop}
	The gradient of $L(\Sigma(B,C))$ 
	with respect to $(B,C)$ can be computed as follows, 
	\begin{align*} 
		\nabla_B(L(\Sigma(B,C))) &= 2\Sigma(B,C)\Sigma(B^t, \nabla L), \\ 
		\nabla_C (L(\Sigma(B,C))) &= 2 \Sigma(B^t, \nabla L),
	\end{align*}
	where $\nabla L$ denotes the gradient of $\Sigma \mapsto L(\Sigma)$. 
\end{prop}
\begin{proof}
Similar to \citet{Malago2018} we differentiate  
the Lyapunov equation  and  we obtain:  
\begin{align*} 
	&B \frac{\partial \Sigma(B,C)}{\partial B_{ij}} + 
\frac{\partial \Sigma(B,C)}{\partial B_{ij}} B^t + Q_{(i,j)}(B,C) = 0, \\
	&Q_{(i,j)}(B,C) = E_{(i,j)}\Sigma(B,C) + \Sigma(B,C) E_{(j,i)}, 
\end{align*}
where 
	$\left(E_{(i,j)}\right)_{kl} = \delta_{ik}
\delta_{jl}$ with $\delta_{ij}$ the usual Kronecker delta.
The Jacobian components  
are thus solutions of  Lyapunov equations, 
\begin{equation}
	\frac{\partial \Sigma(B,C)}{\partial B_{i,j}} = 
	\Sigma\left(B, Q_{(i,j)}(B,C)\right).
	\label{eq:gradB}  
\end{equation}

Thanks to \eqref{eq:gradB} we can compute 
the gradient of any function, which is a composition 
of $\Sigma(B,C)$ and a differentiable function over the cone of positive 
definite matrices \mbox{$L: \mathrm{PD}(p) \to \mathbb{R}$}, as
{\small 
\begin{equation}
	\frac{\partial L\left(\Sigma(B,C)\right)}{\partial B_{ij}} 
	= \tr\left( \Sigma(B,Q_{(i,j)}) 
	\frac{\partial L(\Sigma(B,C))}
    {\partial \Sigma}\right).
	\label{eq:gradb}
 \end{equation}
 }

We note now that, for fixed stable $B$, $\Sigma(B,\cdot)$ is a linear operator 
on the symmetric matrices with adjoint operator given by $\Sigma(B^t, \cdot)$
\citep{bhatia1997}. That is,
\begin{equation*}
	\tr\left( \Sigma(B, C) D \right) = \tr\left( C \Sigma(B^t, D) \right).
\end{equation*}

Thus from \eqref{eq:gradb} we obtain the desired 
expression for the gradient,
\begin{align*}
	\frac{\partial L\left( \Sigma(B,C) \right)}{\partial B_{ij}} = & 
	\tr\left(Q_{(i,j)} \Sigma\left( B^t, \nabla L \right) \right) \\  
   = & \left( 2\Sigma(B,C)\Sigma(B^t,\nabla L) \right)_{ij}. 
\end{align*}
	The formula for $\nabla_C(L(\Sigma(B,C)))$ can be obtained analogously.
\end{proof}
  
\begin{algorithm}[h]
	\caption{Proximal gradient algorithm for minimization of $\ell_1$-penalized loss}
	\label{alg:proxgradb}
	\begin{algorithmic}[1]
		\small
		\REQUIRE  $L:\mathrm{PD}(p) \to \mathbb{R} \text{ differentiable}$, \\
		\hspace{8pt} $B_0 \in \mathrm{Stab}(p)$, \\  
		\hspace{8pt} $M \in \mathbb{N}$,  
		 $\varepsilon, \lambda, \kappa > 0, \alpha \in (0,1)$ 
	        \STATE{$B = B_0$, $C = I_p$}
		\STATE{$\Sigma = \Sigma(B,C)$ }
	        \REPEAT 
		\STATE{$f = L(\Sigma)  + \kappa ||C - I_p||^2_F$}
		\STATE{$g =  \lambda \rho_1(B) $}
		\STATE{$D = \Sigma(B^t, \nabla L)$}
		\STATE{$\nabla_C = 2\operatorname{diag}(D) 
		          + 2\kappa (C - I_p)$}
		\STATE{$\nabla_B = 2\Sigma D$}
		\STATE{$t = \max\{ 0 \leq  u \leq 1 : 
		C - u \nabla_C \in \mathrm{PD}(p) \}$}
                \STATE{$r=\max\{ 0 \leq  u \leq 1 : 
		\mathcal{S}_{u\lambda}(B - u \nabla_B) \in \mathrm{Stab}(p) \}$}
                \STATE{$s = 1$} 
		\LOOP
		\STATE{$B' = \mathcal{S}_{sr\lambda}(B - s r \nabla_B)$ }
		\STATE{$C' = C - st \nabla_C$}
		\STATE{$\Sigma' = \Sigma(B', C')$ \label{line:le}}
		\STATE{$f' = L(\Sigma') + \kappa ||C - I_p||^2_F$}
		\STATE{$g' =  \lambda \rho_1(B') $}
		\STATE{$\nu =\frac{1}{2s}(\frac{1}{r}||B - B'||_F^2 + 
		\frac{1}{t}||C - C'||_F^2)$ \\  
		\hspace{0.2in} $+ \tr( (B' - B)\nabla_B) + \tr((C' - C) \nabla_C) $}
                \IF{$f' + g' \leq f + g$ \AND   $f' \leq f + \nu $ }
		\STATE{\textbf{break}}
		\ELSE 
		\STATE{$s = \alpha s$}
		\ENDIF
		\ENDLOOP 
		\STATE{$\delta = (f + g - f' - g') $}
		\STATE{$\Sigma = \Sigma', B = B', f =f'$}
		\UNTIL{$k > M$ \OR $\delta < \varepsilon$}
		\ENSURE   $B, C, \Sigma$ such that $\Sigma = \Sigma(B,C)$
	\end{algorithmic}
\end{algorithm}

The Lyapunov equations are solved by the Bartels-Stewart
algorithm~\citep{bartels1972} as implemented in
LAPACK~\citep{anderson1999}.  The Bartels-Stewart algorithm consists
of computing the Schur decomposition of the matrix $B$ and then
solving a simplified equation by back-substitution. Observe that
to solve the 
additional Lyapunov equation in the gradient equation the Schur decomposition 
of $B$ can be used and thus it is only computed once in
each iteration (in line~\ref{line:le} in Algorithm \ref{alg:proxgradb}).  Moreover, it is immediate to check
the stability of $B$ from the diagonal elements of its
Schur canonical form.  The run time complexity of one step of the
Algorithm~\ref{alg:proxgradb} is thus $\mathcal{O}(p^3)$.

%

\subsection{REGULARIZATION PATHS} 

As for lasso, \citep{friedman2010}, and graphical lasso,
\citep{friedman2007}, problem~\eqref{eq:l1penalB} is to be solved for
a sequence of regularization parameters
$\lambda_1 < \lambda_2 < \ldots < \lambda_k$. We have implemented the
natural continuation algorithm where the solution $(B_{i-1}, C_{i-1})$ for
$\lambda = \lambda_{i-1}$ is used as initial value of
Algorithm~\ref{alg:proxgradb} for $\lambda = \lambda_i$. Note,
however, that contrary to e.g. \texttt{glmnet}, \citep{friedman2010},
our continuation algorithm starts from a dense estimate and moves
along the regularization parameters in increasing order toward sparser
and sparser solutions. There is no immediate reason for this choice as
the regularization path could be computed, in principle, from sparse 
to dense solutions as in the classical lasso and graphical lasso paths.  
However we empirically observed that better results were obtained 
using an increasing sequence of regularization parameters.


\subsection{DIRECT LASSO PATH} 
\label{sec:lasso}

\cite{fitch2019} suggests estimating $B$ as a sparse,
approximate solution to the Lyapunov equation for
$\Sigma$ fixed and equal to the empirical covariance
matrix, $\hat{\Sigma}$. For fixed $\lambda$ the estimate is the solution to the lasso problem
\begin{equation}
 \begin{array}{ll}
	 \text{minimize}  & \|B\hat{\Sigma} + \hat{\Sigma} B^t + C\|_F^2 + 
	 \lambda\rho_1(B).
\end{array}
	\label{eq:lassoB}
\end{equation}
for a fixed $C$. 
In \cite{fitch2019} all the entries of 
the $B$ matrix are actually penalized, and not only the off-diagonal 
entries as in Equation~\eqref{eq:lassoB}. 

The resulting
\emph{direct lasso path} for a sequence of regularization parameters
can be computed easily by either coordinate descent, \citep{friedman2010}, or 
least angle regression, \citep{efron2004}. 


%
%
%
%
%
%

\section{SIMULATIONS} 

We carried out a simulation study to evaluate the performance of our
proposed estimator and algorithm. The metrics used focus on recovery
of the underlying oriented part of the graph. Performance was
evaluated for Algorithm~\ref{alg:proxgradb} using the negative Gaussian
log-likelihood (\texttt{mloglik-inf} and \texttt{mloglik-0.01}) 
as well as the Frobenius loss
(\texttt{frob-inf}). 
For \texttt{mloglik-inf} and \texttt{frob-inf} we fixed $C=I_p$ (that is, 
$\kappa = +\infty$) while for \texttt{mloglik-0.01} we fixed $\kappa = 0.01$ in 
Algorithm~\ref{alg:proxgradb}. 
The obtained paths were compared to the results
for the direct lasso path (\texttt{lasso}), the graphical lasso (\texttt{glasso}) for
undirected structure recovery~\citep{friedman2007}, and the simpler
covariance thresholding method (\texttt{covthr})~\citep{sojoudi2016}.

Each GCLM was generated by simulating a stable matrix $B$ with entries 
$B_{ij} = \omega_{ij} \varepsilon_{ij}$ for $i\neq j$ and $B_{ii} = - \sum_{j \neq i} 
| B_{ij}| - | \varepsilon_{ii} |$ where $\omega_{ij} \sim \text{Bernoulli}(d)$ 
and $\varepsilon_{ij} \sim N(0,1)$. Moreover, we generated diagonal $C$ matrices 
with $C_{ii} \sim \text{Uniform([0,1])}$. 
Note that each such $(B, C)$ pair has a corresponding
mixed-graph $\mathcal{G}$ whose only bidirected edges are $i
\leftrightarrow i$ and whose directed edges are generated
independently and with uniform probability $d$. 


We generated models of sizes $p = 10, \ldots, 100$ and with
edge probabilities 
$d = \frac{k}{p}$ with $k \in 
\{1,2,3,4\}$. 
For each pair $(p, k)$ we generated $100$ GCLMs
as described above and applied the different structure recovery methods using 
$N = 1000$ observations from a multivariate Gaussian
distribution with covariance matrix solving the Lyapunov equation. 


To further explore the stability of the structure recovery under different levels 
of marginalization, 
we considered
the problem of recovering the directed part of the graph $\mathcal{G}[10]$ for the 
first $10$ coordinates. 
This simulation scenario corresponds to marginalized models,  
as described in Section~\ref{sec:marginalization}.

\subsection{DETAILS OF THE COMPARED METHODS} 

For each method but \texttt{covthr} we obtained a solution path along
a log-regular sequence of $100$ regularization parameters
\[0  < \frac{\lambda_{\max}}{10^{4}} = 
\lambda_1 <  \ldots < \lambda_{100} = \lambda_{\max}.\] 
For our methods we 
used $\lambda_{\max} = 6$. For \texttt{lasso}, $\lambda_{\max}$ was the  
 smallest 
penalization parameter such that 
the matrix $B$ was diagonal. For \texttt{glasso}, 
$\lambda_{\max} = \max\{ \hat{\Sigma}_{ij}  \}$, 
resulting in a path similar to the default 
in the {glasso} R package, \citep{glasso}. 
For covariance thresholding (\texttt{covthr}) we 
obtained instead a solution path by thresholding the absolute values
in the sample covariance matrix at its off-diagonal entries. 

In Algorithm~\ref{alg:proxgradb} the relative convergence tolerance was
$\varepsilon = 10^{-4}$, the maximum number of iterations 
was $M = 100$ and $\alpha = 0.5$. 

\begin{figure}[h]
	\centering
	\includegraphics[scale=1]{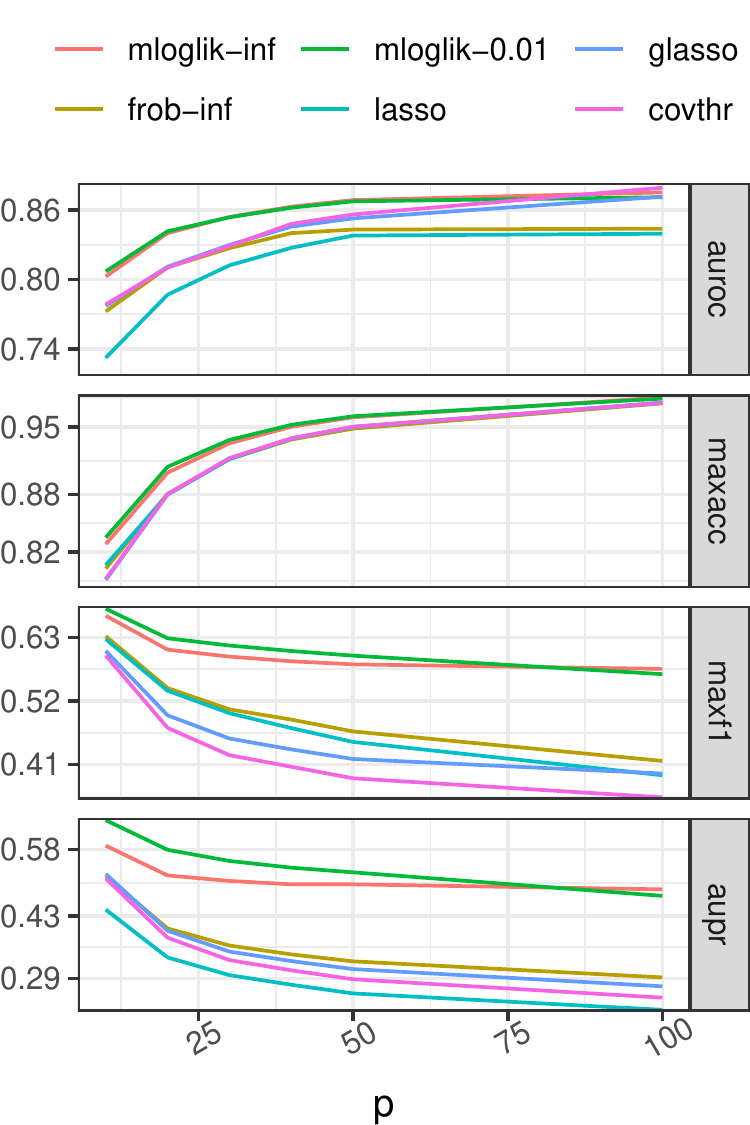}
	\caption{Structure recovery simulation results.
		Average evaluation metrics (rows) as a function of the model size  
	for different algorithms (colors).
	}
	\label{fig:simulated}
\end{figure}


Data was standardized, which means that all methods used the 
empirical correlation matrix, $\hat{R}$, of the sample, 
and for \texttt{lasso} we fixed
$C$ to the identity matrix. Finally, Algorithm~\ref{alg:proxgradb} 
was initialized
with the stable and symmetric matrix 
$B_0 = -\frac{1}{2} \hat{R}^{-1}$ fulfilling $\hat{R} = \Sigma(B_0,
I_p)$. 

\begin{figure}
	\centering
	\includegraphics[scale=1]{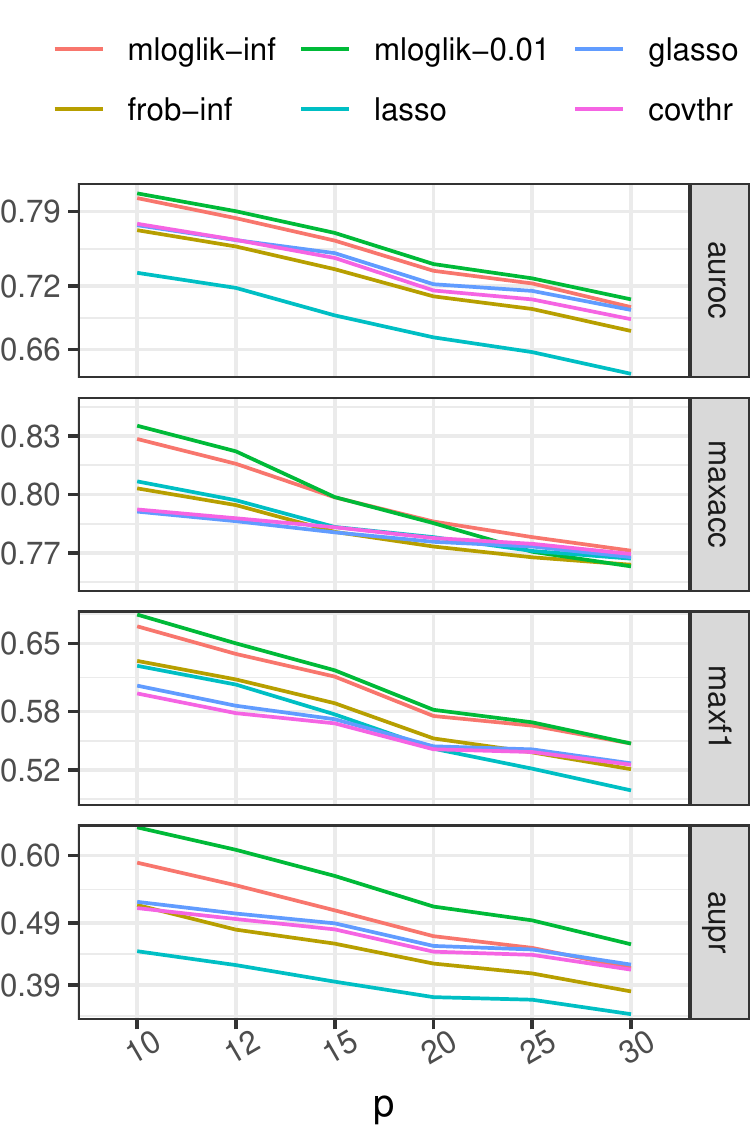}
	\caption{
		Recovery of marginalized model simulation results.
		Average evaluation metrics (rows) as a function of the model size  
	for different 
	algorithms (colors).
	}
	\label{fig:marginalized}
\end{figure}

\begin{figure}
        \centering
	\includegraphics[scale = 1]{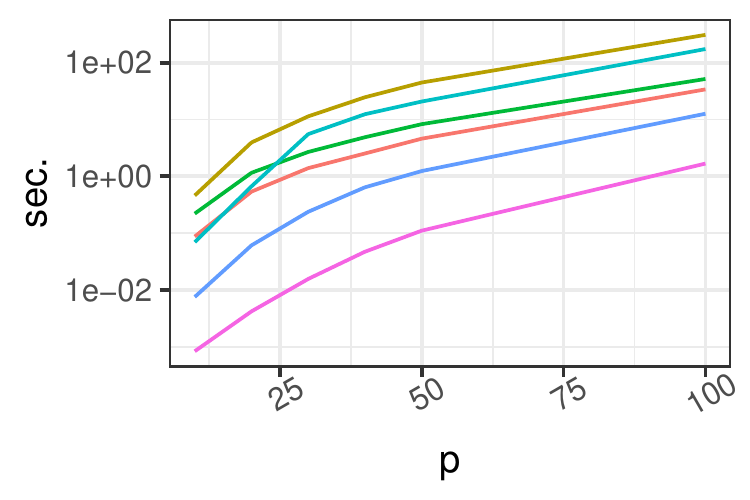}
	\caption{Average run times as a function
	of the system size ($p$) for different methods (colors).}
	\label{fig:times}
\end{figure}

\subsection{RESULTS} 

Each method gives a solution path of graphs for a sequence of 
regularization parameters. We computed the following metrics to 
evaluate the methods: 
\begin{itemize}
	\item The path-wise maximum accuracy of edge recovery (\texttt{maxacc}). 
	\item The path-wise maximum F1 score (\texttt{maxf1}).   
	\item The area under the ROC curves (\texttt{auroc}), 
		obtained as the 
		true positive rate vs the false positive rate
		for each value of the regularization parameter. 
	\item The area under the precision-recall curves (\texttt{aupr}), 
		obtained as the
		precision vs the recall for each value of the 
		regularization parameter.  
\end{itemize}
All the above metrics were computed considering the graph recovery as a
classification problem over the $p(p-1)$ off-diagonal elements 
of the adjacency matrix. 
In particular, undirected graphs obtained with the  
methods \texttt{glasso} and \texttt{covthr} are 
evaluated as directed graphs where each undirected edge is 
translated into the two possible directed edges. 


Figure~\ref{fig:simulated}  shows the results from the simulation
experiments averaged over the $100$
repetitions and the different edge densities,  
Figure~\ref{fig:marginalized} shows the results from the simulation 
experiment with marginalized models. 

From Figure~\ref{fig:simulated} we observe 
that among our proposed methods, using the negative log-likelihood was
always
better than the Frobenius loss. 
Across all simulations, \texttt{mloglik-inf} and \texttt{mloglik-0.01} 
were clearly superiors to the other methods with respect to all
our evaluation metrics. For these two methods the evaluations were 
highly similar with the exception of the precision-recall curve where
\texttt{mloglik-0.01} obtained consistently higher results, especially in the 
recovery of marginalized models. 
Moreover, we observe that \texttt{frob-inf}  was superior to
\texttt{lasso} in the recovery of the true graph with respect to almost all
the metrics.  




In Figure~\ref{fig:times} the average run times of the different methods are 
reported. We observe that there is practically no difference in the run
times between fixing $C=I_p$ (\texttt{mloglik-inf}) and allowing 
the estimation of a diagonal $C$ matrix (\texttt{mloglik-0.01}). 
Also it is interesting to note that the run time of the \texttt{lasso} 
method is equal to the \texttt{mloglik} methods for large systems. 
While \texttt{frob-inf} requires approximately one order of magnitude more
time to reach convergence (or the maximum number of iterations) then 
\texttt{mloglik-inf}. Given that each iteration of
Algorithm~\ref{alg:proxgradb} is computationally more expensive using 
the negative log-likelihood than the Frobenius loss, we deduce that
\texttt{frob-inf} requires in general a much higher number of iterations 
to converge. 

\section{PROTEIN-SIGNALING NETWORKS}

We apply the proposed method with log-likelihood loss to the flow-cytometry 
data in \cite{sachs2005} containing observations of
$11$ phosphorylated proteins and
phospholipids from $n = 7466$ cells.  
Data were recorded under nine different conditions consisting of 
nine different
stimulatory and inhibitory interventions. 

We apply the following procedure, 
inspired by stability selection methods~\citep{meinshausen2010}.

\begin{enumerate}
	\item Randomly split the observations in two subsets with 
		the same cardinality: \textit{Train} and \textit{Test}. 
	\item Apply Algorithm~\ref{alg:proxgradb} using the 
		estimated correlation matrix from \textit{Test},
		to obtain the estimated $B$ 
		matrices along a regularization path. 
	\item Fit the maximum-likelihood estimators 
		(using a minor modification of 
		Algorithm~\ref{alg:proxgradb} with $\lambda = 0$) 
		for all the structures obtained in the previous point. 
	\item Select the structure that obtains the maximum likelihood 
		with respect to the empirical covariance 
		matrix of \textit{Test}. 
\end{enumerate}
After repeating $200$ times the above selection based on random-splitting  
we compute the number of times each edge was selected.  

Figure~\ref{fig:sachs} shows the resulting graph obtained by retaining
directed edges appearing in at least $85\%$ 
of the repetitions.

We observe that the method retrieves edges  consistent with the 
ground truth of 
conventionally accepted interactions~\citep{sachs2005, meinshausen2016}. 
In particular, the estimated graph in 
Figure~\ref{fig:sachs} contains  
8 of the 18 edges reported in \cite{sachs2005}, among them: the regulatory 
interactions between PKA and Mek, p38, Erk; the relationships JNK $\leftarrow$ PKC $\rightarrow$ p38; and PLC $\rightarrow$ PIP2 $\leftarrow$ PIP3.  
We observe that our model estimate also some cycles, in particular the 
interactions PLC $\leftrightarrow$ PIP2, JNK $\leftrightarrow$ PKC $\leftrightarrow$
P38 and Mek $\leftrightarrow$ Raf which have been recovered in the literature
by other approaches~\citep{meinshausen2016}.


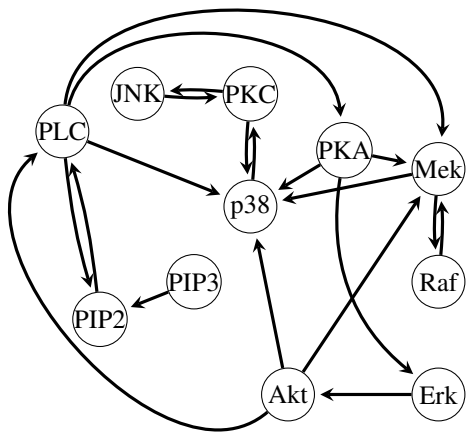
\begin{figure}
	\centering
\tikzset{vertex/.style = {shape=circle,draw,minimum size=2em, inner sep = 0pt}}
\tikzset{vertexdot/.style = {shape=circle,fill=red,color=red,draw,
	minimum size=0.5em, inner sep = 0pt}}
\tikzset{vertexhid/.style = {shape=rectangle,draw,
	minimum size=1.5em, inner sep = 0pt}}
\tikzset{edge/.style = {->,> = latex', thick}}
\tikzset{edgeun/.style = {-, thick}}
\tikzset{edgebi/.style = {<->,> = latex', thick}}
\tikzset{
	node/.style={circle,inner sep=1mm,minimum size=1.2cm,
	draw,very thick,black,fill=red!20,text=black, scale = 0.7},
	nondirectional/.style={very thick,black},
	unidirectional/.style={nondirectional,shorten >=2pt,-stealth},
	bidirectional/.style={unidirectional,bend right=5}
}
\begin{tikzpicture}[scale=5]
	\node [vertex] (v1) at (1.000000, 0.300000)	{Raf};
	\node [vertex] (v2) at (1.000000, 0.600000)	{Mek};
	\node [vertex] (v3) at (0.000000, 0.700000)	{PLC};
	\node [vertex] (v4) at (0.100000, 0.200000)	{PIP2};
	\node [vertex] (v5) at (0.350000, 0.300000)	{PIP3};
	\node [vertex] (v6) at (1.000000, 0.000000)	{Erk};
	\node [vertex] (v7) at (0.600000, 0.000000)	{Akt};
	\node [vertex] (v8) at (0.750000, 0.650000)	{PKA};
	\node [vertex] (v9) at (0.500000, 0.800000)	{PKC};
	\node [vertex] (v10) at (0.500000, 0.500000)	{p38};
	\node [vertex] (v11) at (0.200000, 0.800000)	{JNK};

	\path [bidirectional] (v2) edge (v1);
 	\path [bidirectional] (v1) edge (v2);
	\path [unidirectional] (v2) edge (v10);
	\path [unidirectional, bend left=90] (v3) edge (v2);
	\path [unidirectional, bend left=80] (v3) edge (v8);
	\path [unidirectional] (v3) edge (v10);
	\path [bidirectional] (v4) edge (v3);
 	\path [bidirectional] (v3) edge (v4);
	\path [unidirectional] (v5) edge (v4);
	\path [unidirectional] (v6) edge (v7);
	\path [unidirectional] (v7) edge (v2);
	\path [unidirectional, bend left=90] (v7) edge (v3);
	\path [unidirectional] (v7) edge (v10);
	\path [unidirectional] (v8) edge (v2);
	\path [unidirectional, bend right] (v8) edge (v6);
	\path [unidirectional] (v8) edge (v10);
	\path [bidirectional] (v10) edge (v9);
 	\path [bidirectional] (v9) edge (v10);
	\path [bidirectional] (v11) edge (v9);
 	\path [bidirectional] (v9) edge (v11);

\end{tikzpicture}
	\caption{Estimated graph from data in \cite{sachs2005}.
	Self loops and 
	bidirected edges are not plotted.}
	\label{fig:sachs}
\end{figure}

\section{DISCUSSION}

We have presented a novel graphical model yielding a parametrization
of covariance matrices via solutions of the continuous Lyapunov equation with 
parameter matrices $(B,C)$ compatible with a given mixed
graph. Using a trek representation and a graphical projection we showed that also
marginalized models can be parametrized by the continuous Lyapunov
equation. 

We investigated the performance of learning the directed part of the
graph via penalized loss minimization where we fixed $C$ to be a
diagonal matrix. A
similar approach was considered by \cite{fitch2019} where, moreover, the
matrix $C$ was fixed as the identity $I_p$. As shown in 
Section~\ref{sec:marginalization}, marginalization may
result in the $C$ matrix being increasingly misspecified and 
non-diagonal, thus 
the general deterioration of the performances for 
\texttt{mloglik-inf}, \texttt{mloglik-0.01}, 
\texttt{frob-inf} and \texttt{lasso} as in our
simulation experiment is to be expected. 

It was pivotal for our implementation of the proximal gradient
algorithm that gradients for the loss functions could be computed as
efficiently as possible. This was achieved via the representation of
the Jacobian of $\Sigma(B, I)$ via Lyapunov equations and exploiting 
the adjoint of the linear operator $\Sigma(B,\cdot)$.
When compared to the direct lasso path as proposed by
\cite{fitch2019}, our methods are computationally comparable, 
and even faster for larger systems, it appears. Moreover, our simulation experiment showed that minimizing
the $\ell_1$-penalized negative log-likelihood resulted in a more
efficient estimator of the directed part of the graph than using the 
Frobenius loss. 

\subsection{FUTURE DIRECTIONS}


One open problem is to estimate $C$ as a non-diagonal, but sparse,
matrix corresponding to the bidirected edges of the graph. This is
particularly interesting when we consider data from a marginalized
model. 
Imposing an additional penalty of the type $\lambda \rho_1(C)$ 
the corresponding proximal  
gradient-step is easily implemented to jointly estimate
sparse matrices $(B,C)$. However, the optimization problem 
becomes highly non-convex, and initial experiments suggest that 
the algorithm is easily trapped in local minima. We conjecture that 
these computational problems are closely related to the 
fundamental open problem of determining the joint 
identifiability of the $B$ and $C$ parameters from $\Sigma$. It is
ongoing work to provide answers to such identifiability questions and to
devise algorithms that are able to jointly estimate $B$ and $C$.

\subsection{REPRODUCIBILITY}  

Instructions and source files to replicate the examples and the experiments 
can be found at 
\url{https://github.com/gherardovarando/gclm_experiments}.
An R package is available from \url{https://github.com/gherardovarando/gclm},
implementing 
Algorithm~\ref{alg:proxgradb}. 

\subsubsection*{Acknowledgements}
The authors thank Mathias Drton for insightful discussions and 
feedback.  
This work was supported by VILLUM FONDEN (grant 13358).

\bibliographystyle{plainnat}
\bibliography{biblio}
\end{document}